\documentclass{article}

\usepackage{hyperref}
\usepackage[round]{natbib}
\usepackage{graphicx} % Required for inserting images
\usepackage{amssymb}
\usepackage[dvipsnames]{xcolor}
\usepackage[mathscr]{euscript}
\usepackage{amsthm}
\usepackage{amsmath}
\usepackage{algorithm}
\usepackage{algpseudocode}
\usepackage{graphicx}
\usepackage{graphicx}

\DeclareSymbolFont{rsfs}{U}{rsfs}{m}{n}
\DeclareSymbolFontAlphabet{\mathscrsfs}{rsfs}
\usepackage[english]{babel}

\newtheorem{definition}{Definition}
\newtheorem{lemma}{Lemma}

\usepackage{amsmath,amssymb}

\title{Data Deletion for Linear Regression with Noisy SGD}
\date{}

\author{Zhangjie, Xia \and Chi-Hua, Wang \and Guang, Cheng}
\begin{document}
\maketitle

\begin{abstract}
  In the current era of big data and machine learning, it's essential to find ways to shrink the size of training dataset while preserving the training performance to improve efficiency. However, the challenge behind it includes providing practical ways to find points that can be deleted without significantly harming the training result and suffering from problems like underfitting. We therefore present the perfect deleted point problem for 1-step noisy SGD in the classical linear regression task, which aims to find the perfect deleted point in the training dataset such that the model resulted from the deleted dataset will be identical to the one trained without deleting it. We apply the so-called signal-to-noise ratio and suggest that its value is closely related to the selection of the perfect deleted point. We also implement an algorithm based on this and empirically show the effectiveness of it in a synthetic dataset. Finally we analyze the consequences of the perfect deleted point, specifically how it affects the training performance and privacy budget, therefore highlighting its potential. This research underscores the importance of data deletion and calls for urgent need for more studies in this field.  
\end{abstract}
% \paragraph{Keywords:}
% Data Deletion, Differential Privacy, Gradient Descent, Linear Regression.

\section{INTRODUCTION}
Machine learning (ML) is a field of study in artificial intelligence concerned with the development of statistical algorithms that can learn from given data and generalize to unseen data, thus performing future tasks without explicit instructions. It has been playing an essential role in our everyday life, and has become one of the most rapidly growing research fields nowadays. \cite{jordan2015machine} reviewed the core methods and recent progress in this field and suggested that despite its remarkable success recently, it still emerges numerous research opportunities for many unsolved problems.

Data deletion \citep{garg2020formalizing,ginart2019making,oesterling2024fair} is one of the most intriguing areas among all those research questions. It's widely known that more training data tends to improve the training performance dramatically \citep{ramezan2021effects,foody2006training}. However, training with more data also has serious drawbacks. For instance, \cite{budach2022effects} empirically showed that data quality in training dataset can seriously affect the training performance of many widely used ML algorithms including regression, classification and clustering, therefore more training data has a higher risk of containing poisoning data points. \cite{narciso2020application} suggested that the massive amount of data used by the industry can cause serious energy efficiency related problems, thus increasing time and space complexity during training and boosting energy consumption. \cite{hawkins2004problem} pointed out that too much training data can make the ML model learn noise and unimportant patterns from given data, thus generalize poorly to unseen testing data, which is also referred to as overfitting. 

Additionally, many regulations have been established on behalf of the "Right to be Forgotten" in many countries and states, including \cite{gdpr2016}, and the recent California Consumer Privacy Act Right (CCPA) \citep{ccpa2024} allows users to require companies like Google, Facebook, Twitter, etc to forget and delete these personal data to protect their privacy. They can also ask these platforms to erase outdated information or limit the retention of personal details, such as photos and emails, to a specific time period. However, since personal data is often collected incrementally, the process of deleting it in chronological order poses significant challenges for machine learning models. This issue highlights the need for the development and analysis of deletion-aware online machine learning methods capable of handling such requests efficiently.

Therefore, since more training data doesn't necessarily lead to better result and may be prohibited in some cases, it's natural to ask the following question:
\begin{center}
    \textit{Can we use a smaller training dataset without sacrificing training accuracy?}
\end{center}
Data deletion, or machine unlearning \citep{cao2015towards,bourtoule2021machine}, addresses this problem by selecting certain training samples to be deleted which doesn't harm the training performance and results in identical final models. Further work proposed many solutions to this question in certain scenarios \citep{baumhauer2022machine,golatkar2020eternal,graves2021amnesiac} and we refer the readers to Section 2 for more details.

This paper proposes an innovative idea to the question of data deletion in one step of noisy stochastic gradient descent (SGD) for linear regression task. Section 2 and 3 presents related works and the basic problem formulation. Section 4 introduces our statistical method for solving data deletion problem and shows our algorithm correspondingly. Section 5 highlights the major consequences of our method, specifically how it affects empirical risk and model privacy. Section 6 empirically shows the use of perfect deleted point in the synthetic dataset. Section 7 summarizes the paper and suggests some possible future work.

\paragraph{Our Contributions:}
\begin{enumerate}
    \item Propose a hypothesis testing method to solve the data deletion problem in linear regression using noisy SGD.
    \item Show that the deletion of our perfect deleted point affects the training performance the least compared with other points potentially.
    \item Show the deletion of our perfect deleted point possibly posts the minimum privacy issue compared with other points.
    \item Empirically demonstrate the effectiveness and potential of perfect deleted point in the synthetic dataset.
\end{enumerate}

\section{RELATED LITERATURE}
\paragraph{Noisy gradient descent.}
Our research contributes to the study of machine learning security, particularly focusing on noisy gradient descent (Noisy GD) \citep{avella2023differentially,das2023beyond,wu2020noisy}. Noisy GD is best exemplified by DP-SGD, which is proposed by \cite{abadi2016deep}. The basic idea of it is to introduce noise into the gradient descent process to provide privacy guarantees during training. In the DP-SGD framework, each update is equipped with a privacy burden, which quantifies the privacy leakage in each step. There are many causes for the privacy leakage, and our research shows how the deletion of one data point in the training dataset can affect the privacy budget in Section 5.2. We conclude that there exists certain data point, which we call the perfect deleted point, that will cause least privacy budget than other points, and highlights potential future work on how the behaviour of training data can affect machine learning security.
\paragraph{Data deletion.}
The key of our research is to understand data deletion in a certain framework and provide a reliable approach to solve this problem. Despite the boosting progress in this field, we have found that most researches use singular metric, like accuracy, as the dominating factor to help them find points to be deleted \citep{wu2020deltagrad,guo2019certified,garg2020formalizing,neel2021descent,gupta2021adaptive}. Obviously this is \textbf{ad hoc}, leading to unreliable and inconsistent results. Our research address this issue by applying various metrics, namely the distribution of model weights, training loss and privacy budget, to verify that our perfect deleted point is indeed the "perfect" one to be deleted among all data points in the training dataset, thus enhancing the reliability of our method and providing a new standard for future researchers working in this domain. We believe our contribution will significantly improve the use of theoretical tools available for data deletion, fostering greater trust and robustness in this area.
 \paragraph{Membership inference attack.} 
 Our research approach is inspired by the current progress in the area of membership inference attack (MI attack). In such an attack, the malicious group wants to predict whether a certain data point belongs to the training dataset or the whole data distribution, given either final model output (black-box attack) or information about model structure in the whole training process (white-box attack) \citep{shokri2017membership}. While our initial goal is to find the perfect deleted point, equivalently we want to find a data point from the training dataset such that the adversary is most likely to fail to perform a MI attack for that point, i.e., the data point is somehow redundant and makes limited contribution to the training of the model. Many researchers have contributed to this problem from different perspectives, and we are inspired by \cite{leemann2023gaussianmembershipinferenceprivacy} and \cite{8429311}'s works specifically. \cite{8429311} analyzes MI attack using the concept of membership advantage, and inspired by this, we apply the adjusted concept of absolute membership advantage to find the perfect deleted point. \cite{leemann2023gaussianmembershipinferenceprivacy} analyzes MI attack from a hypothesis testing perspective, and similarly, we find the perfect deleted point with the smallest absolute membership advantage using likelihood ratio test. We believe the concept of membership advantage and the approach of hypothesis testing may yield more future work in the problem of data deletion.

Here we also recite two concepts from \cite{yeom2018privacyriskmachinelearning} to prepare the readers for our future discussion. The first one is about membership advantage.
\begin{definition}
    (Membership experiment $Exp^M(A, D)$ \citep{yeom2018privacyriskmachinelearning}): Let $A$ be an adversary, $S$ be the training dataset and $D$ be a probability distribution over data points $(x,y)$. The experiment proceeds as follows:
    \begin{enumerate}
        \item Sample $S \sim D^n$.
        \item Draw $b \in \{0,1\}$ uniformly at random.
        \item Draw $z \sim S$ if $b=0$, or $z \sim D$ if $b=1$.
        \item $A$ tries to predict value of $b$ and outputs either $0$ or $1$. If it makes the right prediction, i.e., its output is equal to $b$, then $Exp^M(A, D)=1$. Otherwise, $Exp^M(A, D)=0$.
    \end{enumerate}
\end{definition}
In other words, the adversary $A$ tries to predict whether data point $z$ is drawn from the training dataset or from the whole data distribution in this experiment. If it predicts correctly, then the experiment is successful and outputs 1; otherwise it fails and outputs 0. Based on the experiment, we further define membership advantage as follows:
\begin{definition}
    (Membership advantage \citep{yeom2018privacyriskmachinelearning}): The membership advantage of point $v$ is defined as
    \begin{center}
        $Adv(v)= Pr[A=0|b=0]-Pr[A=0|b=1]$.
    \end{center}
\end{definition}
Equivalently, the membership advantage can be understood as the difference between true negative rate (TNR), i.e., $Pr[A=0|b=0]$, and false negative rate (FNR), i.e., $Pr[A=0|b=1]$, of the membership experiment $Exp^M(A,D)$. Intuitively, if TNR is close to FNR, then it would be hard to tell whether $z$ is drawn from $S$ or $D$, which means the null and alternative hypothesis give similar probability distributions. We later leverage this idea to find the perfect deleted point.

Another important concept we need from \cite{yeom2018privacyriskmachinelearning} is $\epsilon$-differential privacy, where $\epsilon$ is called privacy budget and large $\epsilon$ means serious concerns about training data confidentiality.
\begin{definition}
    (Differential privacy \citep{yeom2018privacyriskmachinelearning}): An algorithm $A:X^n \to Y$ satisfies $\epsilon$-differential privacy if for all $S,S' \in X^n$ that differ in a single value, the following holds:
    \begin{equation}
        Pr[A(S) \in Y] \leq e^\epsilon Pr[A(S') \in Y].
    \end{equation}
\end{definition}
The following Lemma characterizes the connection between membership advantage and privacy budget, which will be used in Section 5.2.
\begin{lemma}
    (Upper bound of membership advantage for $\epsilon$-differentially private algorithm \citep{yeom2018privacyriskmachinelearning}) Let A be an $\epsilon$-differentially private algorithm, then we have
    \begin{equation}
        Adv(A) \leq e^{\epsilon}-1.
    \end{equation}
\end{lemma}

\section{PROBLEM SETUP}
\subsection{Perfect deleted point problem}
Assume we are working on two different datasets: $D_0$, referred to as the \textit{original} dataset, and $D_1$, known as the \textit{deleted} dataset. Our goal is to find a point $v^* \in D_0$, which is called the \textit{perfect deleted point}, such that the deletion of it will cause the minimum change on the final model, i.e., the distribution of model weights resulted from deleting $v^*$ is identical to the one when not deleting it, compared with all the other points. 

However, it can be difficult and inefficient to directly find the perfect deleted point. Intuitively, for each data point $v_i \in D_0$, we need to delete it from the original dataset $D_0$ and use the deleted dataset $D_0 \setminus v_i$ to train a new ML model. Among all the $n$ models we get from deleting each $v_i$, we want to find the data point $v^* \in D_0$ such that using the deleted dataset $D_1=D_0 \setminus v^*$, we will obtain the most similar model, i.e., model weights, to the one trained with the original dataset $D_0$, compared with deleting all the other $n-1$ data points. We would need to repeat the same deleting-and-training process for $n$ times if we want to directly find the perfect deleted point, which is unnecessary as we propose an alternative approach in this paper that doesn't require any additional deleting-and-training at all. We will discuss it in Section 4.2 in detail.

We are motivated to find the perfect deleted point because previous studies have shown that despite improvement of training performance using a larger training dataset \citep{ramezan2021effects,foody2006training}, too much training data can cause overfitting if the model becomes too complex and learns noise or irrelevant patterns rather than generalizing well \citep{ying2019overview,hawkins2004problem}. Also, it will be more computationally-friendly and data-efficient if using a smaller training dataset \citep{acun2021understanding,zhu2016we}. 
\subsection{Assumptions}
 In this paper, we propose a novel method for finding the perfect deleted point with the following assumptions:
\begin{enumerate}
    \item The original dataset $D_0$ contains finitely many data points, i.e., $n$ is a positive integer that is not infinity.
    \item We are considering the classical linear regression task, i.e., point $v_i=(x_i,y_i) \in D_0$ where $x_i$ is the feature vector and $y_i$ is the label of corresponding $x_i$. Thus, the loss function of point $v_i$ for a ML model with model weights $w$ is 
    \begin{equation}
        l(w,v_i)=(y_i-\langle w,x_i \rangle)^2.
    \end{equation}
    The empirical risk given model weight being $w$ and dataset $D=\{(x_i,y_i)\}_{i=1}^n$ is defined as 
    \begin{equation}
        L(w;D)=\frac{1}{n}\sum_{i=1}^nl(w;(x_i,y_i)),
    \end{equation}
    which evaluates the discrepancy between predicted outcomes and true outcomes of the whole dataset. The mean is taken over all individual losses here because it's a common approach to use mean gradient in noisy SGD defined below.
    \item The ML model uses noisy gradient descent to minimize the empirical risk defined in equation (4) and update its weights in each step, i.e., \begin{equation}
        w_1=w_0-\gamma(\nabla L(w;D)+\eta)
    \end{equation} 
    where $\eta=N(0,\sigma^2I)$ is the Gaussian noise and $w_0,w_1$ are model weights before and after noisy gradient descent respectively. Note we only consider one step of noisy gradient descent in this paper.
\end{enumerate}

\section{KEY THEORETICAL RESULTS}
In Section 4, we present our theoretical approach on how to find the perfect deleted point $v^*$ from a hypothesis testing perspective and design a concrete algorithm to realize this idea.
\subsection{Notations}
 With the linear regression setting in Section 3.2, we first formally define how we can obtain the deleted dataset from the original dataset.
\begin{definition}
    (Original dataset and deleted dataset): Denote $D_0=\{(x_i,y_i)\}_{i=1}^n$ as the original training dataset and
    \begin{center}
        $D_1=D_0 \setminus \{v\}$
    \end{center}
    as the deleted training dataset with the deleted point $v=(x_v,y_v) \in D_0$.
\end{definition}
Given Definition 4 and equation (4), we next define the empirical risk and gradient for original and deleted dataset respectively.
\begin{definition}
    (Original empirical risk and its gradient): Given an original dataset $D_0$ defined in Definition 4. For a model weight being $w$, denote the individual loss of $w$ on training sample $(x_i,y_i)$ as $l(w;(x_i,y_i))$. Then the original empirical risk $L(w;D_0)$ is defined as 
    \begin{center}
        $L(w;D_0)=\frac{1}{n}\sum_{i=1}^nl(w;(x_i,y_i))$.
    \end{center}
    In addition, $\nabla_wL(w;D_0)$ is called the original gradient of model weight w on dataset $D_0$.
\end{definition}
\begin{definition}
    (Deleted empirical risk and its gradient): Given a deleted dataset $D_1$ defined in Definition 4. For a model weight being $w$, denote the individual loss of $w$ on training sample $(x_i,y_i)$ as $l(w;(x_i,y_i))$. Then the deleted empirical risk $L(w;D_1)$ is defined as 
    \begin{center}
        $L(w;D_1)=\frac{1}{n-1}[\sum_{i=1}^nl(w;(x_i,y_i))-l(w;(x_v,y_v)]$.
    \end{center}
    In addition, $\nabla_wL(w;D_1)$ is called the deleted gradient of model weight w on dataset $D_1$ with deleted point being $v$.
\end{definition}
Lemma 2 describes a basic relationship of original gradient and deleted gradient.
\begin{lemma}
    (Identity between original and deleted gradient): 
    \begin{equation}
        \nabla_wL(w;D_1)=\frac{n}{n-1}\nabla_wL(w;D_0)-\frac{1}{n-1}\nabla_wl(w;v).
    \end{equation}
\end{lemma}
\begin{proof}
    This identity is a direct consequence of Definition 5 and 6.
\end{proof}
\subsection{Strategy to find the perfect deleted point}
 Since we only consider one step of noisy gradient update here and the initial weight $w_0$ doesn't change whether or not we delete $v$, from equation (5) we know that in order to get similar model weights $w_1$ for original dataset $D_0$ and deleted dataset $D_1$, we need $\Delta w=w_1-w_0$ to be similar, i.e., have identical probability distributions. Equivalently, it means TNR and FNR considering the distributions of $D_0$ and $D_1$ should be close to each other, i.e., membership advantage of our perfect deleted point $v^*$ should be close to 0.
\begin{lemma}
    (Absolute membership advantage of $v$):
    \begin{equation}
        |Adv(v)|=|\Phi(\Phi^{-1}(1-\alpha)-d_v)-\alpha|
    \end{equation}
    where $\mu(D)=-\gamma\nabla_w L(w;D)$, $\sigma_\gamma=\gamma\sigma$, $d_v=\frac{||\mu(D_1)-\mu(D_0)||_2}{\sigma_\gamma}$ is the signal-to noise ratio for point $v$ and $\alpha$ is the Type I error.
\end{lemma}
\begin{proof}
    See Appendix A.1.
\end{proof}
From Lemma 3, we notice that the absolute membership advantage depends on Type I error $\alpha$ and signal-to-noise ratio $d_v$, and Type I error $\alpha$, which by definition is the false positive rate (FPR), i.e., the probability of rejecting null hypothesis when it's actually true, is therefore equivalent to our manually-picked significance level. So next we try to give an explicit formula for how to calculate signal-to-noise ratio $d_v$ by applying Lemma 2 to help us find the perfect deleted point.
\begin{lemma}
   (Signal-to-noise ratio $d_v$ for $v=(x_v,y_v) \in D_0$):
   \begin{equation}
       d_v=\frac{||y_vx_v-S_{yx}+S_{xx}w-x_vx_v^Tw||_2}{\sqrt{\frac{\gamma(n-1)}{2}}\sigma}
   \end{equation}
   where $S_{yx}=\frac{1}{n}\sum_{i=1}^ny_ix_i$, $S_{xx}=\frac{1}{n}\sum_{i=1}^nx_ix_i^T$
\end{lemma}
\begin{proof}
    See Appendix A.2.
\end{proof}
From Lemma 3 and 4, we know we need to find data point $v$ such that $d_v$ calculated from it using equation (8) will minimize the absolute membership advantage $|Adv(v)|$. By simple calculation, we know $d_v$ should be close to $\Phi^{-1}(1-\alpha)-\Phi^{-1}(\alpha)=2\Phi^{-1}(1-\alpha)$ so that $|Adv(v)|$ will be minimized, i.e., close to 0. Therefore, the perfect deleted point problem can be reformulated as: we want to find point $v^* \in D_0$ such that
\begin{center}
    $v^*=\mathop{\arg\min}\limits_{v \in D_0}|d_v-2\Phi^{-1}(1-\alpha)|$
\end{center}

Next we formally propose Algorithm 1 to find the perfect deleted point in the original dataset $D_0$. To find the perfect deleted point, we need to first obtain signal-to-noise ratio $d_v$ for each data point $v \in D_0$ by applying Lemma 4. Then by Lemma 3, we choose a significance level and use it as the Type I error $\alpha$. Finally for each data point $v$, we compare the distance between $d_v$ and $2\Phi^{-1}(1-\alpha)$, and the one with the smallest distance gives the perfect deleted point $v^*$. Note we introduce a hyper-parameter $\delta$ here to quantify the tolerance we can have for the largest acceptable distance, i.e., if the smallest distance $|d_v-2\Phi^{-1}(1-\alpha)|$ for all $v \in D_0$ is bigger than $\delta$, then we say there's no perfect deleted point in the dataset. So now instead of directly training $n$ models separately using deleted datasets and comparing their model weights to the original model weights for each data point, our approach only needs to compute signal-to-noise ratio $d_v$ and compares its value to $2\Phi^{-1}(1-\alpha)$ for all data points in one training step, which only takes $O(n)$ time and $O(n)$ space for one step of noisy SGD training, where $n$ is the size of the original dataset. Thus our approach for finding the perfect deleted point is more time and space efficient than the direct approach. 

\begin{algorithm}
\caption{Find Perfect Deleted Point}
\label{alg:Find Perfect Deleted Point}
\begin{algorithmic}[1]
\State \textbf{Input:} $D_0=\{v_i\}_{i=1}^n$, $\gamma$, $\sigma^2$, initial weight $w$, tolerance $\delta$, significance level $\alpha'$
\State \textbf{Output:} $v^*$
\State $S_{yx}, S_{xx} \leftarrow \frac{1}{n}\sum_{i=1}^ny_ix_i, \frac{1}{n}\sum_{i=1}^nx_ix_i^T$
\State $d \leftarrow \delta$
\State $\alpha \leftarrow \alpha'$
\State $v^* \leftarrow \texttt{None}$
\For{\texttt{each} $v=(x_v,y_v) \in D_0$}
    \State $d_v \leftarrow \frac{||y_vx_v-S_{yx}+S_{xx}w-x_vx_v^Tw||_2}{\sqrt{\frac{\gamma(n-1)}{2}}\sigma}$ 
    \State $d_1 \leftarrow |d_v-2\Phi^{-1}(1-\alpha)|$
    \If{$d_1 <= d$}
        \State $d \leftarrow d_1$
        \State $v^* \leftarrow v$
    \EndIf
\EndFor
\State \Return $v^*$
\end{algorithmic}
\end{algorithm}

\section{MAIN INSIGHTS}
In Section 5, we further discuss how the perfect deleted point will affect the training performance (Section 5.1) and how it will cause privacy concern in the training process (Section 5.2). It provides insightful analysis on why the perfect deleted point should be of interest.
\subsection{Analysis of training loss for perfect deleted point}
In order to understand how perfect deleted point will affect the training performance, we first propose the following concept of absolute membership error that quantifies how different a point $v$ is from being the perfect deleted point.
\begin{definition}
    (Absolute membership error): For point $v \in D_0$ with corresponding signal-to-noise ratio $d_v$, we define its absolute membership error as
    \begin{equation}
        \epsilon_v = d_v - 2\Phi^{-1}(1-\alpha).
    \end{equation}
\end{definition}
Based on this definition, absolute membership error of the perfect deleted point $\epsilon_{v^*}$ should be close to 0 by analysis in Section 4.2. From now on, we denote $\epsilon_{v^*}$ by $\epsilon^*$ for brevity. We propose the following Lemma to bound the change of training loss for deleting point $v$.
\begin{lemma}
    (Bounds for change of empirical risk for deleting point $v$):
    Assume change of empirical risk, i.e., $L(w;D_1)-L(w;D_0)$, is non-negative.
    \begin{enumerate}
        \item Given point $v=(x_v,y_v)$ and its absolute membership error $\epsilon_v$, change of empirical risk after deleting $v$ is lower bounded by 
    \begin{equation}
        \frac{1}{n-1}L(w;D_0)-[\epsilon_v+2\Phi^{-1}(1-\alpha)]C-\frac{||S_{yx}-S_{xx}w||_2}{(n-1)||x_{v}||_2}
    \end{equation}
    and upper bounded by 
    \begin{equation}
        \frac{1}{n-1}L(w;D_0)+[\epsilon_v+2\Phi^{-1}(1-\alpha)]C-\frac{||S_{yx}-S_{xx}w||_2}{(n-1)||x_{v}||_2}
    \end{equation}
    where $C=\frac{\sigma}{||x_{v}||_2}\sqrt{\frac{\gamma}{2(n-1)}}$.
    \item Suppose features of the training data $\{(x_i,y_i)\}_{i=1}^n$ has the following property: $||x_i||_2 \geq B, \forall i \in [1,n]$ where $B$ is a positive number. Then the lower bound can be generalized to
    \begin{equation}
        \frac{1}{n-1}L(w;D_0)-[\epsilon_v+2\Phi^{-1}(1-\alpha)]D-\frac{||S_{yx}-S_{xx}w||_2}{(n-1)B}
    \end{equation}
    and the upper bound can be generalized to 
    \begin{equation}
        \frac{1}{n-1}L(w;D_0)+[\epsilon_v+2\Phi^{-1}(1-\alpha)]D-\frac{||S_{yx}-S_{xx}w||_2}{(n-1)B}
    \end{equation}
    where $D=\frac{\sigma}{B}\sqrt{\frac{\gamma}{2(n-1)}}$.
    \end{enumerate}
    Note both bounds in the first and second part of this Lemma are non-negative because deletion of samples always increases training risk. 
\end{lemma}
\begin{proof}
    See Appendix B.1.
\end{proof}
The novel part of Lemma 5 is that given several candidates for the perfect deleted point, i.e.,  they all give very similar $d_v$, the one with the smallest $l_2$ feature norm, i.e., $||x_v||_2$, is always better because the lower bound of change of training loss is decreased if we decrease $||x_{v}||_2$. The upper bound is also decreased if we decrease $||x_{v}||_2$ because individual loss $l(w;v)$ is non-negative, meaning the part after $\frac{1}{n-1}L(w;D_0)$, which is $-\frac{1}{n-1}l(w;v)$, is therefore negative. So decreasing the denominator $||x_v||_2$ will decrease the upper bound as well. In other words, deleting the one with the smallest $l_2$ feature norm will improve training performance the most potentially because it has  smaller both lower and upper bound for increment of empirical risk.

Also, given the same training dataset, i.e., same $B$, the perfect deleted point gives the tightest bound for change of loss after deletion compared with other points because $\epsilon^* \leq \epsilon_v$. In other words, we can have an accurate estimate of how the perfect deleted point will affect training performance by setting $\epsilon^* = 0$ before even finding the exact perfect deleted point.
\subsection{Relation of privacy budget $\epsilon$ for perfect deleted point}
In this section, we try to understand how the perfect deleted point will bring about privacy issues in the training process.
\begin{lemma}
    (Lower bound for privacy budget for deleting point $v$): Privacy budget $\epsilon$ after deleting point $v$ with absolute membership error $\epsilon_v$ is at least $max\{ln[\Phi(\Phi^{-1}(\alpha)-\epsilon_v)+1-\alpha],0\}$.
\end{lemma}
\begin{proof}
    See Appendix B.2. 
\end{proof}
From Lemma 6, we observe that assume privacy budget must be non-negative, then as long as the absolute membership error $\epsilon_v$ is non-negative, it can achieve the smallest lower bound $0$ for privacy budget. In other words, if we have several candidates for the perfect deleted point, i.e., they all give very similar $d_v$, then the one with positive absolute membership error will be considered as the best option for the perfect deleted point because it may potentially distort privacy guarantees to the minimum degree compared with deleting other points. Note we don't know exactly whether it'll cause the smallest privacy budget because we only know it can achieve the smallest lower bound for privacy budget, so only theoretically can it lead to a smaller privacy budget.

\section{EXPERIMENT}
Section 6 exemplifies the effectiveness our approach by conducting an empirical experiment to find the perfect deleted point in a synthetic dataset. By comparing the model weights in different scenarios, we show that our perfect deleted point can indeed change the distribution of model weights to the minimal extent.

\subsection{Experiment setup}
To better visualize how the perfect deleted point will affect the model weights, we generate a 2D synthetic dataset with 200 samples in this case as the original dataset $D_0$, where the $x$ coordinate of all these points is randomly picked within the range $[0,10]$ with seed $40$. To get the $y$ coordinate, we assume the true model behind our generated dataset is $y=3.1415926535x$ since we are dealing with classic linear regression here. Then we use Gaussian noise with mean $0$ and standard deviation $2$, amplifies its magnitude by $10$, and add it to the true $y$ coordinate to generate the synthetic $y$ coordinate. The scatterplot of our generated dataset looks like Figure 1. Although this is a simple dataset that may not mimic the real world dataset, it suffices to analyze the perfect deleted point problem on it in this paper as it satisfies all the assumptions in the given linear regression scenario.
\begin{figure}[h]
\centering
\includegraphics[width=\linewidth]{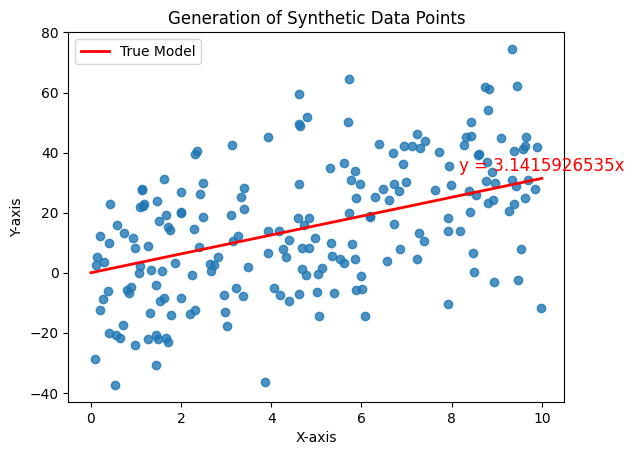}
\caption{Synthetic Dataset}
\end{figure}

The hyperparameters we used in this experiment are listed here: initial weight $w_0=0$, learning rate $\gamma =0.01$, Gaussian noise $\eta=N(0,4)$, Type I error $\alpha=0.01$, and tolerance $\delta=100$. Since this is a stochastic algorithm and we want to get the distribution of weights, we run the 1-step noisy SGD for 100 iterations to understand the overall behaviour of it.

\subsection{Result analysis}

To understand the effect of perfect deleted point, we first focus on the distribution of model weights after deleting it and compare it with the distribution of model weights when we don't delete any point. We use the distribution of model weights when we randomly delete a data point as a benchmark to better demonstrate the potential of perfect deleted point compared with all the other points. Secondly, we repeat the same finding-and-deleting process for multiple steps of noisy SGD to augment the effect of the perfect deleted point as deletion of one data point may not affect much compared with no deletion and random deletion because of the size of original dataset $D_0$ mitigates the impact.
\begin{figure}[ht]
\centering
\includegraphics[width=\linewidth, height=9cm]{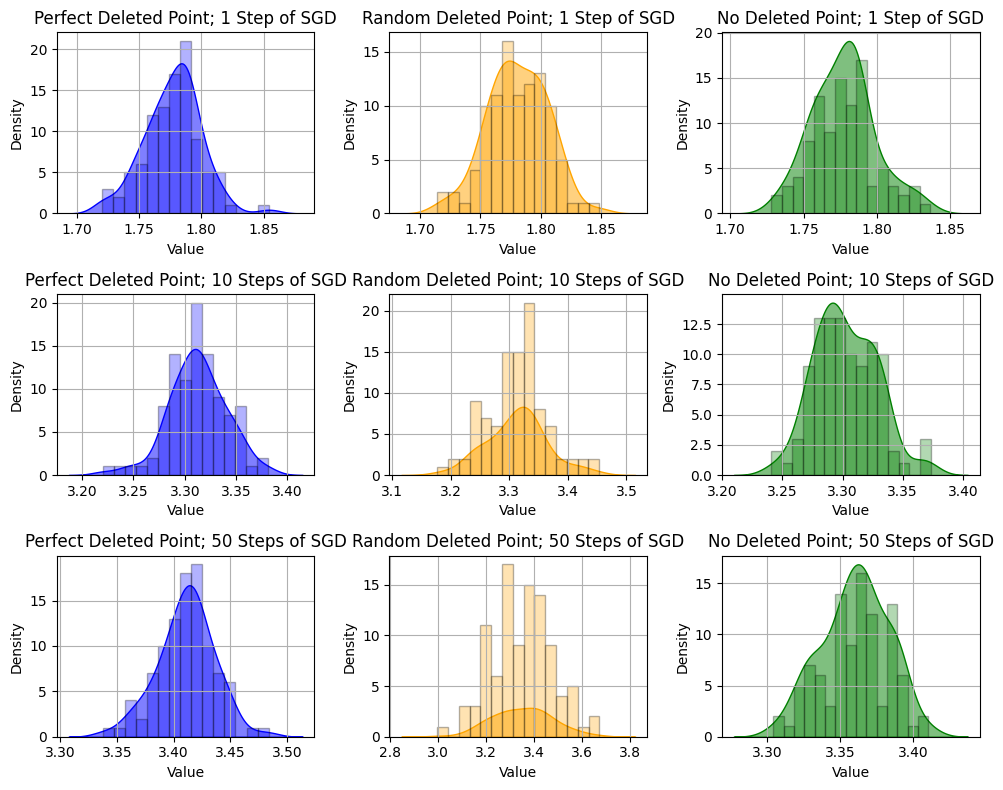}
\caption{Distribution of model weights after 1, 10, 50 steps of noisy SGD using perfect deleted point, randomly deleted point and no deleted point in 100 iterations. The histogram plots the occurrence of model weights in each bin, and the shaded area is the kernel density estimate (KDE) of the distribution.}
\end{figure}

\begin{table*}[ht]
\caption{Statistics for Distributions of Model Weight} \label{Mean of Model Weight}
\begin{center}
\begin{tabular}{lll}
\textbf{MODEL}  &\textbf{MEAN}  &\textbf{VARIANCE}\\
\hline \\
1-step; perfect deleted point         &1.77693  &0.00051\\
1-step; random deleted point         &1.78043  &0.00061\\
1-step; no deleted point         &1.77718  &0.00046\\ \hline
10-step; perfect deleted point         &3.31226  &0.00079\\
10-step; random deleted point         &3.31290  &\textcolor{red}{0.00265}\\
10-step; no deleted point         &3.30248  &0.00066\\ \hline
50-step; perfect deleted point         &\textcolor{red}{3.41079}  &0.00063\\
50-step; random deleted point         &3.35108  &\textcolor{red}{0.01560}\\
50-step; no deleted point         &3.36008  &0.00051\\
\end{tabular}
\end{center}
\end{table*}

Figure 2 and Table 1 demonstrate the distribution of model weights in different cases both visually and statistically. From the first row Figure 2, we see in one step of noisy SGD, all three cases demonstrate a similar normal distribution, which is proved by the mean and variance of the first three lines in Table 1. It is because the deletion of a single data point in $D_0$ which contains 200 samples only posts a trivial effect on the distribution of model weights. From the next two rows of Figure 2, we see that as the number of noisy SGD steps increases, the weight distribution of perfect deleted point resembles much more closely to no deleted point compared with random deleted point. This means that our perfect deleted point is indeed “perfect” in the sense of preserving model weight after deletion compared with other points. When we look at the highlighted numbers in the third column of Table 1, we can see that as number of noisy SGD steps increases, the variance of distribution of model weights using random deleted point grows larger, which means the normal distribution becomes "flatter" and we are more likely to get random weights after more random deletion. But by the convergence of noisy SGD, we know this is not acceptable, which again shows the advantage of perfect deleted point.

% \begin{table*}[ht]
% \caption{Statistics for Distributions of Model Weight} \label{Mean of Model Weight}
% \begin{center}
% \begin{tabular}{lll}
% \textbf{MODEL}  &\textbf{MEAN}  &\textbf{VARIANCE}\\
% \hline \\
% 1-step; perfect deleted point         &1.77693  &0.00051\\
% 1-step; random deleted point         &1.78043  &0.00061\\
% 1-step; no deleted point         &1.77718  &0.00046\\ \hline
% 10-step; perfect deleted point         &3.31226  &0.00079\\
% 10-step; random deleted point         &3.31290  &\textcolor{red}{0.00265}\\
% 10-step; no deleted point         &3.30248  &0.00066\\ \hline
% 50-step; perfect deleted point         &\textcolor{red}{3.41079}  &0.00063\\
% 50-step; random deleted point         &3.35108  &\textcolor{red}{0.01560}\\
% 50-step; no deleted point         &3.36008  &0.00051\\
% \end{tabular}
% \end{center}
% \end{table*}

However, if we look at the highlighted number in the second column of Table 1, we see that despite having small variance, the mean of distribution of model weights after 50 steps using perfect deleted point shifts tremendously compared with both no and random deleted point, which gives a slightly different model weight distribution compared with no deleted point. To address this problem, we increase the Type I error $\alpha$ to 0.05, and Figure 3 plots the corresponding result. 
\begin{figure}[h]
\centering
\includegraphics[width=\linewidth, height=8.2cm]{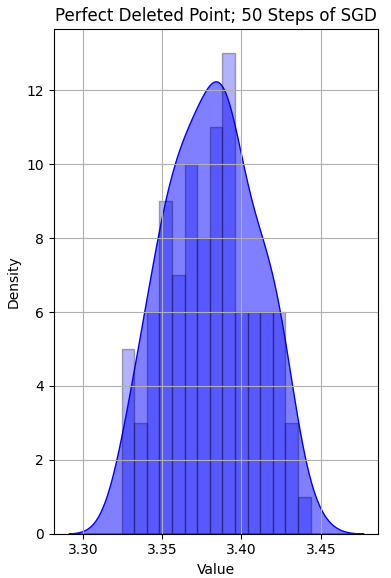}
\caption{Distribution of model weights after 50 steps of noisy SGD using perfect deleted point when $\alpha=0.05$, with mean = 3.38091 and variance = 0.00077.}
\end{figure}
As we can see, the mean becomes much smaller and the whole distribution is therefore much similar to the one with no deleted point. An intuitive explanation behind it may be that as we increase Type I error $\alpha$, we increase significance level as well, which allows us more room to make mistakes using perfect deleted point in many steps of noisy SGD, i.e., deleting more points, thus improving statistical power of our perfect deleted point. This trade-off between $\alpha$ and mean may lead to potential future work and needs more rigorous investigation.
% \begin{figure}[h]
% \centering
% \includegraphics[width=\linewidth, height=8.2cm]{3.png}
% \caption{Distribution of model weights after 50 steps of noisy SGD using perfect deleted point when $\alpha=0.05$, with mean = 3.38091 and variance = 0.00077.}
% \end{figure}

\section{CONCLUSION}
In conclusion, our research proposes the perfect deleted point problem in one step of noisy SGD for classical linear regression. We present a hypothesis testing approach to solve this problem using signal-to-noise ratio and construct an algorithm based on it. Moreover, we analyze how the perfect deleted point is related to training performance and differential privacy, thus showing the consequence of the perfect deleted point and further demonstrating the importance of this problem. Finally, we create a synthetic dataset and implement our algorithm on it and show the minimum impact perfect deleted point will have on the distribution of model weights. The implications of our work demonstrate the urgent necessity for strategies of data deletion in different scenarios to improve model training efficiency. Moving forward, our research sets a new standard for understanding and analyzing such problems, ensuring the efficiency and security for training of ML models. Some potential future work includes analyzing data deletion for the whole SGD training process and for the general regression task.

\renewcommand{\bibname}{References}
\renewcommand{\bibsection}{\section*{\bibname}}
\bibliographystyle{unsrtnat}
\nocite{*}
\bibliography{ref.bib}

\section*{Checklist}
 \begin{enumerate}

 \item For all models and algorithms presented, check if you include:
 \begin{enumerate}
   \item A clear description of the mathematical setting, assumptions, algorithm, and/or model. [Yes] See Section 4
   \item An analysis of the properties and complexity (time, space, sample size) of any algorithm. [Yes] See Section 4
   \item (Optional) Anonymized source code, with specification of all dependencies, including external libraries. [Yes] See Appendix
 \end{enumerate}

 \item For any theoretical claim, check if you include:
 \begin{enumerate}
   \item Statements of the full set of assumptions of all theoretical results. [Yes] See Section 3
   \item Complete proofs of all theoretical results. [Yes] See Appendix
   \item Clear explanations of any assumptions. [Yes] See Section 3     
 \end{enumerate}

 \item For all figures and tables that present empirical results, check if you include:
 \begin{enumerate}
   \item The code, data, and instructions needed to reproduce the main experimental results (either in the supplemental material or as a URL). [Yes] See Appendix
   \item All the training details (e.g., data splits, hyperparameters, how they were chosen). [Yes] See Section 6
         \item A clear definition of the specific measure or statistics and error bars (e.g., with respect to the random seed after running experiments multiple times). [Yes] See Section 6
         \item A description of the computing infrastructure used. (e.g., type of GPUs, internal cluster, or cloud provider). [Not Applicable]
 \end{enumerate}

 \item If you are using existing assets (e.g., code, data, models) or curating/releasing new assets, check if you include:
 \begin{enumerate}
   \item Citations of the creator If your work uses existing assets. [Not Applicable]
   \item The license information of the assets, if applicable. [Not Applicable]
   \item New assets either in the supplemental material or as a URL, if applicable. [Not Applicable]
   \item Information about consent from data providers/curators. [Not Applicable]
   \item Discussion of sensible content if applicable, e.g., personally identifiable information or offensive content. [Not Applicable]
 \end{enumerate}

 \item If you used crowdsourcing or conducted research with human subjects, check if you include:
 \begin{enumerate}
   \item The full text of instructions given to participants and screenshots. [Not Applicable]
   \item Descriptions of potential participant risks, with links to Institutional Review Board (IRB) approvals if applicable. [Not Applicable]
   \item The estimated hourly wage paid to participants and the total amount spent on participant compensation. [Not Applicable]
 \end{enumerate}

 \end{enumerate}

\clearpage
\appendix
\onecolumn
\section{Proof sketch for section 4}
\subsection{Proof of Lemma 3}
Proof of Lemma 3 follows the same logic as Section 4.4 of \cite{scharf1991statistical}. It is also adopted from Section A.2 Proof of Lemma 4 
 in \cite{wang2024badgdunifieddatacentricframework}. However, here we calculate the absolute membership advantage instead of the trade-off function as in \cite{wang2024badgdunifieddatacentricframework}. We reinstate the proof here for readers' reference and change the final step according to our need.

We use binary hypothesis testing considering distributions of $\Delta w=w_1-w_0$ of $D_0$ and $D_1$, denoted as $\Delta w(D_0)$ and $\Delta w(D_1)$,  as defined in Section 3.2 to calculate FNR and TNR, therefore calculating membership advantage of v.
    \begin{equation}
    H_0: \Delta w(D_0) \sim N(\mu(D_0),\sigma_\gamma^2I),   H_1: \Delta w(D_1) \sim N(\mu(D_1),\sigma_\gamma^2I)
    \end{equation}
    Define $W=(\sigma_\gamma^2I)^{-1}(\mu(D_1)-\mu(D_0))$, then log likelihood ratio has the form 
    \begin{equation}
        L(\Delta w)=W^T \Delta w
    \end{equation}
    by equation (4.29) in \cite{scharf1991statistical}. Equivalently, the binary hypothesis testing of (14) can be changed into 
    \begin{equation}
        H_0: L(\Delta w) \sim N(-\frac{d_v^2}{2},d_v^2),   H_1: L(\Delta w) \sim N(\frac{d_v^2}{2},d_v^2)
    \end{equation}
    where 
    \begin{equation}d_v^2=W^T (\sigma_\gamma^2I)W=(\mu(D_1)-\mu(D_0))^T(\sigma_\gamma^2I)^{-1}(\mu(D_1)-\mu(D_0))
    \end{equation}
    is called the signal-to-noise ratio.

    At significance level $\alpha$, which is also the Type I error, we know (16) has Type II error 
    \begin{equation}
    \beta=1-\Phi(\Phi^{-1}(1-\alpha)-d_v).
    \end{equation}
    Since $TNR$ is probability of failing to reject null hypothesis given null hypothesis is true, which is equal to $1-\alpha$, then 
    \begin{equation}
    |Adv(v)|=|TNR-FNR|=|1-\alpha-\beta|=|\Phi(\Phi^{-1}(1-\alpha)-d_v)-\alpha|
    \end{equation}
    as desired.
\subsection{Proof of Lemma 4}

Proof of Lemma 4 follows similar logic as Section B.6 Proof of Lemma 10 in \cite{wang2024badgdunifieddatacentricframework}. However, here the gradient of loss for the deleted dataset $D_1$ changes, so we need to be careful when we do the computation.
We want to calculate
\begin{equation}
    \begin{aligned}
        d_v&=\frac{||\mu(D_1)-\mu(D_0)||_2}{\sigma_\gamma} \\
        &=\frac{||\nabla_w L(w;D_1)-\nabla_w L(w;D_0)||_2}{\sqrt{\gamma}\sigma}
    \end{aligned}
\end{equation}
Apply Lemma 2, we get
\begin{equation}
    d_v=\frac{||\nabla_w L(w;D_0)-\nabla_w l(w;v)||_2}{\sqrt{\gamma(n-1)}\sigma}
\end{equation}
Given $l(w,(x,y))=(y-\langle w,x \rangle)^2$, the gradient is $\nabla_w l(w,(x,y))=-2(y-\langle w,x \rangle)x$, then
\begin{equation}
    \begin{aligned}
          &\nabla_w L(w;D_0)-\nabla_w l(w;v)\\
        &=\nabla_w \frac{1}{n}\sum_{i=1}^n(y_i-\langle w,x_i \rangle)^2-\nabla_w(y_v- \langle w,x_v \rangle)^2\\
        &=\nabla_w [\frac{1}{n}\sum_{i=1}^ny_i^2-w^T(\frac{2}{n}\sum_{i=1}^ny_ix_i)+w^T(\frac{1}{n}\sum_{i=1}^nx_ix_i^T)w]-\nabla_w[y_v^2-2w^Ty_vx_v+w^Tx_vx_v^Tw]\\
        &=\nabla_w[\frac{1}{n}\sum_{i=1}^ny_i^2-y_v^2]-2\nabla_w[w^T(\frac{1}{n}\sum_{i=1}^ny_ix_i-y_vx_v)]+\nabla_ww^T[\frac{1}{n}\sum_{i=1}^nx_ix_i^T-x_vx_v^T]w\\
        &=2(y_vx_v-\frac{1}{n}\sum_{i=1}^ny_ix_i)+2[\frac{1}{n}\sum_{i=1}^nx_ix_i^T-x_vx_v^T]w\\
        &=2(y_vx_v-S_{yx})+2[S_{xx}-x_vx_v^T]w
    \end{aligned}
\end{equation}
where $S_{yx}=\frac{1}{n}\sum_{i=1}^ny_ix_i,S_{xx}=\frac{1}{n}\sum_{i=1}^nx_ix_i^T$.
Then (21) can be written as 
\begin{equation}
    \begin{aligned}
        d_v&=\frac{||2(y_vx_v-S_{yx})+2[S_{xx}-x_vx_v^T]w||_2}{\sqrt{\gamma(n-1)}\sigma}\\
        &=\frac{||(y_vx_v-S_{yx})+[S_{xx}-x_vx_v^T]w||_2}{\sqrt{\frac{\gamma(n-1)}{2}}\sigma}\\
        &=\frac{||y_vx_v-S_{yx}+S_{xx}w-x_vx_v^Tw||_2}{\sqrt{\frac{\gamma(n-1)}{2}}\sigma}
    \end{aligned}
\end{equation}
\section{Proof sketch for section 5}
\subsection{Proof of Lemma 5}
Since $L(w;D_1)=\frac{n}{n-1}L(w;D_0)-\frac{1}{n-1}l(w;v)$, so we only need to focus on $L(w;D_1)-L(w;D_0)=\frac{1}{n-1}L(w;D_0)-\frac{1}{n-1}l(w;v)=\frac{1}{n-1}L(w;D_0)-\frac{1}{n-1}[y_{v}-\langle x_{v},w \rangle]=\frac{1}{n-1}L(w;D_0)-\frac{1}{n-1}[y_{v}-x_{v}^Tw]$, which is clearly non-negative because $L(w;D_0) > l(w;v)$.
    Based on equation (23), we know
    \begin{equation}
        \begin{aligned}
            (d_v)^2&=(\frac{||y_{v}x_{v}-S_{yx}+S_{xx}w-x_{v}x_{v}^Tw||_2}{\sqrt{\frac{\gamma(n-1)}{2}}\sigma})^2\\
            & \geq \frac{2}{\sigma^2\gamma(n-1)}[||y_{v}x_{v}-x_{v}x_{v}^Tw||_2-||S_{yx}-S_{xx}w||_2]^2 \\
            &=\frac{2}{\sigma^2\gamma(n-1)}[||x_{v}||_2 \cdot ||y_{v}-x_{v}^Tw||_2-||S_{yx}-S_{xx}w||_2]^2.
        \end{aligned}
    \end{equation}
    The inequality is obtained from Triangle Inequality for $l_2$ norm.
    From Definition 7, we also know 
    \begin{equation}
        d_v = \epsilon_v+2\Phi^{-1}(1-\alpha)
    \end{equation}
    So we have
        \begin{center}
            $[\epsilon_v+2\Phi^{-1}(1-\alpha)]^2 \geq \frac{2}{\sigma^2\gamma(n-1)} [||x_{v}||_2 \cdot l(w,v)-||S_{yx}-S_{xx}w||_2]^2$
        \end{center}
        \begin{center}
            $[\epsilon_v+2\Phi^{-1}(1-\alpha)]^2\frac{\sigma^2\gamma(n-1)}{2} \geq [||x_{v}||_2 \cdot l(w,v)-||S_{yx}-S_{xx}w||_2]^2$
        \end{center}
        \begin{center}
            $-[\epsilon_v+2\Phi^{-1}(1-\alpha)]\frac{\sigma}{||x_{v}||_2}\sqrt{\frac{\gamma(n-1)}{2}}+\frac{||S_{yx}-S_{xx}w||_2}{||x_{v}||_2} \leq l(w;v) \leq [\epsilon_v+2\Phi^{-1}(1-\alpha)]\frac{\sigma}{||x_{v}||_2}\sqrt{\frac{\gamma(n-1)}{2}}+\frac{||S_{yx}-S_{xx}w||_2}{||x_{v}||_2}$
        \end{center}
    Multiplying by $\frac{-1}{n-1}$ and adding $\frac{1}{n-1}L(w;D_0)$ on both sides of the inequality, we get the first part of Lemma 5.

    For part 2 of Lemma 5, we further write equation (24) as
    \begin{equation}
        \begin{aligned}
            (d_v)^2& \geq \frac{2}{\sigma^2\gamma(n-1)}[||x_{v}||_2 \cdot ||y_{v}-x_{v}^Tw||_2-||S_{yx}-S_{xx}w||_2]^2\\
            & \geq \frac{2}{\sigma^2\gamma(n-1)}[B \cdot ||y_{v}-x_{v}^Tw||_2-||S_{yx}-S_{xx}w||_2]^2
        \end{aligned}
    \end{equation}
    since $||x_v||_2 \geq B$ by assumption. Then following the same logic of the first part will we get the result as shown in part 2.
\subsection{Proof of Lemma 6}
Apply Lemma 1 and 7, we get
    \begin{center}
       $\Phi(\Phi^{-1}(1-\alpha)-d)-\alpha \leq e^{\epsilon}-1$
    \end{center}
    So
    \begin{equation}
        \begin{aligned}
            \epsilon &\geq ln[\Phi(\Phi^{-1}(1-\alpha)-d)+1-\alpha]
        \end{aligned}
    \end{equation}
    For point $v$, we know $d_v=\epsilon_v+2\Phi^{-1}(1-\alpha)$. So the lower bound for privacy budget $\epsilon$ after deleting perfect deleted point is 
    \begin{equation}
        \begin{aligned}
            ln[\Phi(\Phi^{-1}(1-\alpha)-d)+1-\alpha]&=ln[\Phi(\Phi^{-1}(1-\alpha)-\epsilon_v-2\Phi^{-1}(1-\alpha))+1-\alpha]\\
            &=ln[\Phi(\Phi^{-1}(\alpha)-\epsilon_v)+1-\alpha].\\
        \end{aligned}
    \end{equation}
    Since we assume the privacy budget to be non-negative, we take the $max$ of equation (28) and 0 to obtain the desired result.
\end{document}